\documentclass[12pt]{colt2018}

\usepackage{microtype} 
\usepackage{algpseudocode}
\usepackage{parskip}

\newcommand{\ltdmsa}{\textsc{LMSA}}
\newcommand{\boost}{\textsc{LMSA-Boost}}
\newcommand{\minmax}{\textsc{LMSA-Min-max}}

\newcommand\conf[1]{}
\newcommand\arxiv[1]{#1}

\newcommand\alsoignore[1]{}
\usepackage{paralist}
\usepackage{wrapfig}
\usepackage{mathtools}
\usepackage{amsmath}
\usepackage{bbm}
\usepackage{amsfonts}
\usepackage{amssymb}
\usepackage{fullpage}
\usepackage{graphicx}
\usepackage{latexsym}
\usepackage[mathscr]{euscript}
\usepackage{multirow}
\usepackage{times}

\let\vec\undefined
\usepackage{MnSymbol}
\usepackage{xpatch}

\def\sfB{\mathsf{B}}

\def\sfB{\mathsf{B}}

\newcommand{\disc}{\text{disc}}
\newcommand{\indic}{1}

\DeclareMathOperator*{\E}{\mathbb E}

\DeclareMathOperator*{\argmin}{argmin}

\DeclareMathOperator*{\conv}{conv}

\newtheorem*{rep@theorem}{\rep@title}
\newcommand{\newreptheorem}[2]{%
\newenvironment{rep#1}[1]{%
 \def\rep@title{#2 \ref{##1}}%
 \begin{rep@theorem}}%
 {\end{rep@theorem}}}

\newreptheorem{theorem}{Theorem}
\newreptheorem{lemma}{Lemma}
\newreptheorem{corollary}{Corollary}
\newreptheorem{proposition}{Proposition}
\newreptheorem{example}{Example}
\newreptheorem{properties}{Properties}

\newcommand{\cA}{\mathcal{A}}

\newcommand{\cE}{\mathcal{E}}

\newcommand{\cL}{\mathcal{L}}

\newcommand{\cO}{\mathcal{O}}

\newcommand{\cX}{\mathcal{X}}

\newcommand{\sD}{{\mathscr D}}

\newcommand{\sH}{{\mathscr H}}

\newcommand{\sL}{{\mathscr L}}

\newcommand{\sX}{{\mathscr X}}
\newcommand{\sY}{{\mathscr Y}}

\newcommand{\bI}{{\mathbf I}}

\newcommand{\bm}{{\mathbf m}}

\newcommand{\bR}{{\mathbf R}}

\newcommand{\R}{{\mathfrak R}}
\newcommand{\s}{{\mathfrak s}}

\newcommand{\h}{\widehat}

\newcommand{\wt}{\widetilde}

\renewcommand{\set}[2][]{#1 \{ #2 #1 \} }
\newcommand{\ignore}[1]{}
\newcommand{\norm}[1]{\|#1\|}
\hypersetup{
  colorlinks   = true,
  urlcolor     = blue,
  linkcolor    = blue,
  citecolor   = blue
}

\title{A Theory of Multiple-Source Adaptation with Limited Target Labeled Data}

\coltauthor{\Name{Yishay Mansour}\Email{mansour@google.com}\\
\addr Google Research and Tel Aviv University
\AND
\Name{Mehryar Mohri}\Email{mohri@google.com}\\
\addr Google Research and 
Courant Institute of Mathematical Sciences, New York
\AND
\Name{Jae Ro}\Email{jaero@google.com}\\
\addr Google Research, New York
\AND
\Name{Ananda Theertha Suresh}\Email{theertha@google.com}\\
\addr Google Research, New York
\AND
\Name{Ke Wu}\Email{wuke@google.com}\\
\addr Google Research, New York
}

\begin{document}

\maketitle

\begin{abstract}
  We present a theoretical and algorithmic study of the
  multiple-source domain adaptation problem in the common scenario
  where the learner has access only to a limited amount of labeled
  target data, but where the learner has at disposal a large amount of
  labeled data from multiple source domains. We show that a new family
 of algorithms based on model selection ideas benefits from very
  favorable guarantees in this scenario and discuss some theoretical
  obstacles affecting some alternative techniques. We also report the
  results of several experiments with our algorithms that
  demonstrate their practical effectiveness. 
  \end{abstract}

\section{Introduction}

A common assumption in supervised learning is that training and test
distributions coincide. In practice, however, this assumption often
does not hold. This is because the amount of labeled data available
is too modest to train an accurate model. Instead, the learner must resort
to using labeled samples from one or several alternative source domains 
or distributions that 
are expected to be close to the target domain. How can we leverage
the labeled data from these source domains to come up with an accurate
predictor for the target domain? This is the challenge of the
\emph{domain adaptation problem} that arises in a variety of
different applications, such as in natural
language processing \citep{Blitzer07Biographies,
  dredze2007frustratingly,jiang2007instance}, speech
processing \citep{ gauvain1994maximum, jelinek1997statistical}, and
computer vision \citep{leggetter1995maximum}.

In practice, in addition to a relatively large number of total
labeled data from source domains, the learner also has at %his
disposal a large amount of unlabeled data from the target 
domain, but only little or no data from the
target domain. Various scenarios of adaptation can be distinguished,
depending on parameters including 
the number of source domains, %available,
the presence
or absence of target labeled data, and access to labeled source
data or only to predictors trained on each source domain.

The theoretical analysis of adaptation has been the subject
of several publications in the last decade or so.
The single-source adaptation problem was studied by 
\citet{ben2007analysis} as well as follow-up 
publications \citep{blitzer2008learning} and \citep{ben2010theory},
where the authors presented an analysis in terms of 
a $d_A$-distance, including VC-dimension learning bounds
for the zero-one loss.
\citet{mansour2009domain} and \cite{CortesMohri2011,CortesMohri2014}
presented a general
analysis of single-source adaptation for arbitrary loss
functions, where they introduced the notion of discrepancy,
which they argued is the suitable divergence measure in
adaptation. 
The authors further gave Rademacher complexity
learning bounds in terms of the discrepancy for arbitrary
hypothesis sets and loss functions, as well as pointwise
learning bounds for kernel-based hypothesis sets. 
The
notion of discrepancy coincides with the $d_A$-distance
in the special case of the zero-one loss.

\citet{MansourMohriRostamizadeh2009,MansourMohriRostamizadeh2009a} and \cite{HoffmanMohriZhang2018,HoffmanMohriZhang2020}
considered the \emph{multiple-source adaptation} (MSA) scenario
where the learner has access to unlabeled samples and
a trained predictor for each source domain, with no
access to source labeled data.  This approach has been further
used in many applications such as object recognition
\citep{hoffman_eccv12, gong_icml13, gong_nips13}.
\citet{zhao2018adversarial} and \cite{wen2019domain} considered MSA
with only unlabeled target data available and provided generalization
bounds for classification and regression.

There has been a very large recent literature dealing
with experimental studies of domain
adaptation in various tasks. \citet{ganin2016domain} proposed to learn
features that cannot discriminate between source and target
domains. \citet{tzeng2015simultaneous} proposed a CNN architecture to
exploit unlabeled and sparsely labeled target domain
data. \citet{motiian2017unified}, \cite{motiian2017few} and
\cite{wang2019few} proposed to train maximally separated features via
adversarial learning. \citet{saito2019semi} 
proposed to use a minmax entropy method for domain adaptation. We overview 
more related works in Appendix~\ref{app:related}.

This paper presents a theoretical and algorithmic study of 
multiple-source adaptation (MSA) with limited target labeled
data, a scenario that is similar to the one examined by
\cite{konstantinov2019robust}, who considered the
problem of learning from multiple untrusted sources 
and a single target domain.
We show that a new family
of algorithms based on model selection ideas benefits from very
favorable guarantees in this scenario and discuss some theoretical
obstacles affecting some alternative techniques. We also report the
results of several experiments with our algorithms that
demonstrate their practical effectiveness. 

In Section~\ref{sec:theory}, we introduce some definitions and notation
and formulate our learning problem. 
In Section~\ref{sec:model}, we present 
and analyze our 
algorithmic solutions (\ltdmsa\ algorithms) 
for the adaptation problem considered,
which we prove benefit from
near-optimal guarantees. 
In Section~\ref{sec:disc_all}, 
we discuss some theoretical obstacles affecting some alternative techniques. Then, in
Section~\ref{sec:exp}, we report the results of experiments with our
\ltdmsa\ algorithms and compare them with several other techniques
and baselines.

\section{Preliminaries}
\label{sec:theory}

In this section, we introduce the definitions and notation used in our
analysis and discuss a natural baseline and the formulation of the
learning problem we study.

\subsection{Definitions and notation}

Let $\sX$ denote the input space and $\sY$ the output space. We focus
on the multi-class classification problem where $\sY$ is a finite set
of classes, but much of our results can be extended straightforwardly
to regression and other problems.  The hypotheses we consider are of
the form $h\colon \sX \to \Delta_\sY$, where $\Delta_\sY$ stands for
the simplex over $\sY$. Thus, $h(x)$ is a probability distribution
over the classes or categories that can be assigned to $x \in \sX$. We
denote by $\sH$ a family of such hypotheses.  We denote by $\ell$
a loss function defined over $\Delta_\sY \times \sY$ and taking
non-negative values with upper bound $M$. The loss of $h \in \sH$ for
a labeled sample $(x, y) \in \sX \times \sY$ is given by
$\ell(h(x), y)$. We denote by $\sL_\sD(h)$ the expected loss of a
hypothesis $h$ with respect to a distribution $\sD$ over $\sX \times
\sY$:
\[
\sL_\sD(h) = \E_{(x, y) \sim \sD} [\ell(h(x), y)],
\]
and by $h_\sD$ its minimizer:
$h_\sD = \argmin_{h \in \sH} \sL_\sD(h)$. 

We denote by $\sD_0$ the target domain distribution and by
$\sD_1, \ldots, \sD_p$ the $p$ source domain distributions. During
training, we observe $m_k$ independent samples from distribution
$\sD_k$. We denote by $\h \sD_k$ the corresponding empirical
distribution. We also denote by $m = \sum^p_{k = 1} m_k$ 
the total number of samples observed.  In practice,
we expect $m$ to be significantly larger than $m_0$ ($m \gg m_0$).

It was shown by \citet{mansour2009domain} (see also
\cite{CortesMohri2011}) that the \emph{discrepancy} is the appropriate
divergence between distributions in adaptation. The discrepancy takes
into account the hypothesis set and the loss function, both key
components of the structure of the learning problem. Furthermore, it
has been shown that it can be estimated from finite samples and upper
bounded in terms of other divergences, such as the total variation and
the relative entropy. The discrepancy also coincides with
the $d_\cA$-distance proposed by \citet{ben2007analysis} in the special
case of the zero-one loss.

A finer notion of discrepancy, which we will refer to as the
\emph{label-discrepancy}, was introduced by \citet{mohri2012new},
which is useful in contexts where some target labeled data is
available, as in the problem we are studying here.  For two
distributions $\sD$ and $\sD'$ over $\sX \times \sY$ and a hypothesis
set $\sH$, the label-discrepancy is defined as follows:
\[
\disc_{\sH}(\sD, \sD') = \max_{h \in \sH} | \sL_{\sD}(h) -
\sL_{\sD'}(h)|.
\]
This notion of discrepancy leads to tighter generalization bounds.
When it is small, by definition, the expected loss of any hypothesis
in $\sH$ with respect to a source $\sD$ is close to its expected loss
with respect to $\sD'$. In the rest of the paper, we use label-discrepancy and will refer to it simply by discrepancy.

\subsection{Problem formulation}

%\subsection{Baseline model trained on the target distribution}

What is the best that one can achieve without data from any source
distribution? Suppose we train on the target domain samples
$\h{\sD}_0$ alone, and obtain a model $h_{\h{\sD}_0}$.  By standard
learning theoretic tools \citep{MohriRostamizadehTalwalkar2012}, the
generalization bound for this model can be stated as follows: for
simplicity let the loss be zero-one loss. With probability at least
$1 - \delta$, the minimizer of the empirical risk $\sL_{\h{\sD}_0}(h)$
satisfies,
\begin{equation}
\label{eq:local}
\mspace{-10mu}
\sL_{\sD_0}(h_{\h{\sD}_0}) 
\leq \min_{h \in \sH} \sL_{\sD_0}(h) + \cO
\bigg(\sqrt{\frac{d}{m_0}} + \sqrt{\frac{\log (1/\delta)}{m_0}}
\bigg),
\mspace{-2mu}
\end{equation}
where $d$ is the VC-dimension of the hypothesis class $\sH$. 
For simplicity, we provided generalization bounds in terms of
VC-dimension. They can be easily extended to bounds
based on Rademacher complexity \citep{MohriRostamizadehTalwalkar2012} or pseudo-dimension \citep{pollard2012convergence} for general losses. Finally, there exist distributions and hypotheses where
\eqref{eq:local} is tight \citep[Theorem 3.23]{MohriRostamizadehTalwalkar2012}.

Let $\Delta_p$ be the set of probability distributions over $[p]$.  In
order to provide meaningful bounds and improve upon \eqref{eq:local},
following \citep{MansourMohriRostamizadeh2009a,
  HoffmanMohriZhang2018}, we assume that the target distribution is close to some convex combination of sources in the discrepancy measure, that is, we assume that there is a $\lambda \in \Delta_p$ such that
% \begin{equation*}
%\label{eq:approx_nmf}
 $ \disc_{\sH}(\sD_0,\sD_\lambda)$
%\end{equation*} 
is small, where $\sD_\lambda = \sum^p_{k=1} \lambda_k \sD_k$.

With the above definitions, we can define how good a mixture weight
$\lambda$ is.  For a given $\lambda$, a natural algorithm is to
combine samples from the empirical distributions $\h{\sD}_k$ to obtain
the mixed empirical distribution
$\overline{\sD}_{\lambda} = \sum^p_{k=1} \lambda_k \h{\sD}_k$, and
minimize loss on $\overline{\sD}_{\lambda}$. Let
$h_{\overline{\sD}_\lambda}$ be the minimizer of this loss.  A good
$\lambda$ should lead to $h_{\overline{\sD}_\lambda}$ with the
performance close to that of the optimal estimator for $\sD_0$.  In
other words, the goal is to find $\lambda$ that minimizes
\[
\sL_{\sD_0}(h_{\overline{\sD}_\lambda}) -
  \sL_{\sD_{0}}(h_{\sD_0}).
\]
The above term can be bounded by a uniform excess risk bound as follows:
\begin{align}
& \sL_{\sD_0}(h_{\overline{\sD}_\lambda}) -
  \sL_{\sD_{0}}(h_{\sD_0})
\conf{\nonumber &\\}  
& \leq 2 \max_{h \in \sH} \lvert
  \sL_{\overline{\sD}_\lambda}(h) - \sL_{\sD_\lambda}(h) \rvert + 2
  \disc_{\sH}(\sD_0, \sD_\lambda). \label{eq:disc_lambda}
\end{align}
The derivation of \eqref{eq:disc_lambda} is given in Appendix~\ref{app:disc_lambda}. Let the uniform bound on the
excess risk for a given $\lambda$ be
\begin{equation}
    \label{eq:optimal}
      \cE(\lambda) = 2 \max_{h \in \sH} \lvert
      \sL_{\overline{\sD}_\lambda}(h) - \sL_{\sD_\lambda}(h) \rvert +
      2 \disc_{\sH}(\sD_0, \sD_\lambda),
\end{equation}
and $\lambda^*$ be the mixture weight that minimizes the above uniform
excess bound, i.e.
\[\lambda^* = \argmin_{\lambda \in \Delta_p} \cE(\lambda).
\]
Our goal is to produce a model with error close to $\cE(\lambda^*)$,
without the knowledge of $\lambda^*$.  Before we review the existing
algorithms, we provide a bound on $\cE(\lambda^*)$.

\section{Fixed target mixture}
\label{sec:fixed}
The adaptation problem we are considering can be 
broken down into two parts: 
\begin{inparaenum}[(i)]
\item finding the minimizing mixture weight $\lambda^*$;
\item determining the hypothesis that minimizes the loss over
  corresponding distibution $\h \sD_{\lambda^*}$.
\end{inparaenum} 

In this section, we discuss guarantees for (ii), for a known mixture
weight $\lambda^*$. This will later serve as a reference for our
analysis in the more general case.  More generally, we consider here
guarantees for a fixed mixture weight $\lambda$.

Let $\bm$ denote the empirical distribution of samples
$(m_1/m, m_2/m, \ldots, m_p/m)$. Skewness between distributions is
defined as
$\s (\lambda || \bm) = \sum^p_{k = 1}
\frac{\lambda^2_k}{\bm_k}$. Skewness is a divergence and measures how
far $\lambda$ and the empirical distribution of samples $\bm$ are. It
naturally arises in the generalization bounds of weighted mixtures.
For example, if $\lambda = \bm$, then
$\frac{\s(\lambda || \bm)}{m} = \frac{1}{m}$ and the generalization
bound in Proposition~\ref{lem:known_lambda} will be same as the bound for
the uniform weighted model.  If $\lambda = (1, 0, \ldots, 0)$, then
$\frac{\s(\lambda||\bm)}{m} = \frac{1}{m_1}$ and the generalization
bound will be the same as the bound for training on a single domain. Thus
skewness smoothly interpolates between the uniform weighted model and
the single domain model. For a fixed $\lambda$, the following
generalization bound of \citet{mohri2019agnostic} holds (see also
\cite{blitzer2008learning} in the special case of the zero-one loss).
\begin{proposition}
\label{lem:known_lambda}
Let $\lambda \in \Delta_p$. Then with probability at least
$1 - \delta$, \conf{$\cE(\lambda)$ is bounded by}
\begin{align*}
%\cE(\lambda) 
%& \leq 
\arxiv{ \cE(\lambda)  \leq } 
4M \sqrt{\frac{\s(\lambda || \bm)}{m} \cdot \left(d \log
  \frac{em}{d} + \log \frac{1}{\delta}\right)} 
%  \\
%& \quad 
+  2 \disc_{\sH}(\sD_0, \sD_{\lambda}) .
\end{align*}
\end{proposition}
Since $m \gg m_0$, this guarantee is substantially stronger than 
the bound given for a model trained on the target data
only \eqref{eq:local}.

\section{Unknown target mixture}
\label{sec:model}

Here, we analyze the more realistic scenario where no information
about the target mixture weight is assumed. Our objective is to come
up with a hypothesis whose excess risk guarantee is close to the one
shown in the known target mixture setting.

One natural idea to tackle this problem consists of first determining
the mixture weight $\lambda$ for which $\sD_\lambda$ is the closest to
$\sD_0$ for some divergence measure such as a Bregman divergence
$\sfB$:
\[
\min_{\lambda \in \Delta_p} \sfB(\h{\sD}_0 || \overline{\sD}_\lambda),
\]
But, as discussed in Appendix~\ref{app:bregman}, this approach is
subject to several issues resulting in poor theoretical guarantees. An
alternative consists of seeking $\lambda$ to minimize the following
objective function:
\[
\sL_{\sD_\lambda}(h) + \disc_{\sH}(\sD_0, \sD_\lambda).
\]
However, this requires estimating both the expected loss and the
discrepancy terms and, as discussed in Appendix~\ref{app:convex},
in general, the guarantees for this technique are comparable to
those of the straightforward baseline of training on $\h \sD_0$.

Instead, we will describe a family of algorithms based on a natural
model selection idea, which we show benefits from strong theoretical
guarantees. Unlike the straightforward baseline algorithm or other
techniques just discussed, the dominating term of the learning bounds
for our algorithms are in $\wt O(\sqrt{p/m_0})$, that is the
square-root of the ratio of the number of sources and the number of
target labeled samples and do not depend on the complexity of the
hypothesis set. This is in contrast, for example, to the
$O(\sqrt{d/m_0})$ bound for the straightforward baseline, where $d$ is
the VC-dimension.

We will show that the hypothesis $h_\cA$ returned by our algorithm 
verifies the following inequality:
\begin{align*}
 \sL_{\sD_0}(h_\cA)
\leq \min_{h \in \sH} \sL_{\sD_0}(h) +
 \cE(\lambda^*)  + \tilde{\mathcal{O}} \left( \sqrt{\frac{p}{m_0}} \right).
\end{align*}
We further show that the above additional penalty of
$\tilde{\mathcal{O}} \left( \sqrt{\frac{p}{m_0}} \right)$ is necessary,
by showing an information-theoretic lower bound. We show that for any
algorithm $\cA$, there exists a hypothesis class $\sH$ and domains such
that $\cE(\lambda^*) = 0$ and
\begin{align*}
 \sL_{\sD_0}(h_\cA) \geq  \min_{h \in \sH} \sL_{\sD_0}(h)  + {\Omega} \left( \sqrt{\frac{p}{m_0}} \right).
\end{align*}
These results characterize the penalty term for MSA with limited target data up to logarithmic factors.
We now present our algorithms for the \emph{limited target data MSA} problems:
\ltdmsa, \boost, and
\minmax, as well as an information-theoretic lower bound.

\subsection{\ltdmsa\ algorithm}

Since $\sD_0 \approx \sum_{k} \lambda^*_k \sD_k$, one approach
inspired by model selection consists of determining the hypothesis
with the minimal loss for each value of $\lambda$ and selecting
among them the hypothesis that performs best on $\sD_0$. We call this general algorithm (\ltdmsa)
  (see Figure~\ref{fig:LTDMSA}). 

The algorithm takes as an input a subset $\Lambda$ of $\Delta_p$, which
can be chosen to be a finite cover of $\Delta_p$. For each element of 
$\Lambda$, it finds the
best estimator for $\overline{\sD}_\lambda$, denoted by
$h_{\overline{\sD}_\lambda}$.  Let $\sH_\Lambda$ be the resulting set
of hypotheses. The algorithm then selects the best hypothesis out this
set, by using $\h{\sD}_0$. The algorithm is relatively parameter-free and
straightforward to implement.

\begin{figure}[t]
\centering
\fbox{\begin{minipage}{\arxiv{0.6}\conf{0.45}\textwidth}
\begin{enumerate}

\item For any $\lambda \in \Lambda$, compute $h_{\overline \sD_\lambda}$ defined by
\[
h_{\overline \sD_\lambda} = \argmin_{h \in \sH} \sL_{\overline \sD_\lambda} (h).
\]

\item Define
  $\sH_\Lambda = \{h_{\overline \sD_\lambda} : \forall \lambda \in
  \Lambda\}$.

\item Return $h_m$ defined by
\[
h_m = \argmin_{h \in \sH_{\Lambda}} \sL_{\h{\sD}_0}(h).
\]
\end{enumerate}
\end{minipage}}
\caption{Algorithm \ltdmsa($\Lambda$)}
\label{fig:LTDMSA}
\end{figure}
We now show that \ltdmsa\ benefits from the following favorable
 guarantee, when $\Lambda$ is a finite cover of $\Delta_p$.
\begin{theorem}
\label{thm:model}
Let $\epsilon \leq 1$.  Let $\Lambda$ be a minimal cover of $\Delta_p$
such that for each $\lambda \in \Delta_p$, there exists a
$\lambda_\epsilon \in \Lambda$ such that
$\lVert \lambda - \lambda_\epsilon \rVert_1 \leq \epsilon$. Then, for
any $\delta > 0$, with probability at least $1 - \delta$, the
hypothesis $h_m$ returned by \ltdmsa$(\Lambda)$ satisfies
the following inequality:
\begin{align*}
 \sL_{\sD_0}(h_m) - \min_{h \in \sH} \sL_{\sD_0}(h)  \leq
 \cE(\lambda^*)  
+ 2 \epsilon M +
\frac{2M\sqrt{p \log \frac{p}{\delta\epsilon}}}{\sqrt{m_0}}.
\end{align*}
\end{theorem}
\begin{proof}
We first bound the number of elements in the cover $\Lambda$. Consider
the cover $\Lambda$ given as follows. For each coordinate $k < p$, the
domain weight $\lambda_k$ belongs to the set
$\{0, \epsilon/p, 2\epsilon/p,\ldots, 1\}$, and $(\lambda_\epsilon)_p$
is determined by the fact that $\sum_{k} (\lambda_\epsilon)_k = 1$.
The cover has at most $(p/\epsilon)^{p-1}$ elements and for every
$\lambda$, there is a $\lambda_\epsilon$ such that
$\lVert \lambda - \lambda_\epsilon \rVert_1 \leq \epsilon$.  Hence the
size of the minimal cover is at most $(p/\epsilon)^{p-1}$. Thus, by
McDiarmid's inequality and the union bound, with probability at
least $1 - \delta$, the following holds:
\begin{equation}
\label{eq:finite_bound}
\sL_{\sD_0}(h_m) \leq \min_{h \in \sH_{\Lambda}} \sL_{\sD_0}(h) +
\frac{2M\sqrt{p \log \frac{p}{\delta\epsilon}}}{\sqrt{m_0}}.
\end{equation}
Let $h_\lambda$ denote $h_{\sD_\lambda}$ and $h_{\hat{\lambda}}$
denote $h_{\overline{\sD}_\lambda}$.  For any $\lambda$,
\begin{align}
\conf{&}   \min_{h \in \sH_{\Lambda}} \sL_{\sD_0}(h)
- \min_{h \in \sH} \sL_{\sD_0}(h) \conf{  \nonumber \\}
& \stackrel{(a)}{\leq} \min_{h \in \sH_{\Lambda}} \sL_{\sD_{\lambda}}(h)
- \min_{h \in \sH}\sL_{\sD_{\lambda}}(h)  + 2 \disc_{\sH}(\sD_0, \sD_{\lambda})   \nonumber\\
& \leq  \sL_{\sD_{\lambda}}(h_{\hat{\lambda}_\epsilon}) - \sL_{\sD_{\lambda}}(h_{\lambda}) + 2 \disc_{\sH}(\sD_0, \sD_{\lambda})   \nonumber\\
&  \leq    \sL_{\sD_{\lambda}}(h_{\hat{\lambda}_\epsilon}) -\sL_{\overline{\sD}_{\lambda}}(h_{\hat{\lambda}_\epsilon})
+ \sL_{\overline{\sD}_{\lambda}}(h_{\hat{\lambda}_\epsilon}) - \sL_{\sD_{\lambda}}(h_{\lambda})   \nonumber\conf{\\&} +  2 \disc_{\sH}(\sD_0, \sD_{\lambda})\nonumber \\
&  \stackrel{(b)}{\leq}    \sL_{\sD_{\lambda}}(h_{\hat{\lambda}_\epsilon}) -\sL_{\overline{\sD}_{\lambda}}(h_{\hat{\lambda}_\epsilon})
+ \sL_{\overline{\sD}_{\lambda}}(h_{{\lambda}}) - \sL_{\sD_{\lambda}}(h_{\lambda}) \conf{\nonumber \\&}  + 2\epsilon M +  2 \disc_{\sH}(\sD_0, \sD_{\lambda})   \nonumber \\
& \stackrel{(c)}{\leq}   \cE(\lambda) + 2\epsilon M, \label{eq:temp_1}
\end{align}
$(a)$ follows from the definition of discrepancy and $(c)$ follows from the definition of
$\cE(h)$. For $(b)$, observe that by the definition of $h_\lambda$ and $h_{\lambda_\epsilon}$,
\begin{align*}
    \sL_{\overline{\sD}_{\lambda}}(h_{\hat{\lambda}_\epsilon})
 \conf{&}   \leq \sL_{\overline{\sD}_{\lambda_\epsilon}}(h_{\hat{\lambda}_\epsilon}) + \epsilon M 
\conf{\\&}     \leq \sL_{\overline{\sD}_{\lambda_\epsilon}}(h_{\hat{\lambda}}) + \epsilon M 
 \conf{\\&}    \leq \sL_{\overline{\sD}_{\lambda}}(h_{\hat{\lambda}}) + 2\epsilon M
\conf{\\&}     \leq \sL_{\overline{\sD}_{\lambda}}(h_{{\lambda}}) + 2\epsilon M,
\end{align*}
where the second inequality follows by observing that
$h_{\hat{\lambda}_\epsilon}$ is the optimal estimator for
$\sL_{\overline{\sD}_{\lambda_\epsilon}}$. The last inequality follows
similarly. Combining equations  \eqref{eq:finite_bound}
and \eqref{eq:temp_1} and taking the minimum over $\lambda$ yields the theorem.
\end{proof}
 Note that the
guarantee for \ltdmsa\ is closer to the known mixture setting
when $\lambda^*$ is known. The algorithm finds a mixture weight
$\lambda^*$ that not only admits a small discrepancy with respect to
the distribution $\sD_0$, but also has a small skewness and thus
generalizes better.  In particular, if there are multiple
distributions that are very close to $\sD_0$, then it chooses the one
that generalizes better. Furthermore, if there is a $\lambda^*$ such
that $\sD_0 = \sum^p_{k=1} \lambda^*_k \sD_k$, then the algorithm
chooses either $\lambda^*$ or another $\lambda$ that is slightly worse
in terms of discrepancy, but generalizes substantially better.

Finally, the last term in Theorem~\ref{thm:model},
$\frac{2\sqrt{p \log \frac{p}{\delta\epsilon}}}{\sqrt{m_0}}$, is the
penalty for model selection and only depends on the number of samples
from $\sD_0$ and is independent of $\lambda$. 
Note that for the guarantee of this algorithm to be more favorable than
that of the local model \eqref{eq:local}, we need $p < d$. This,
however, is a fairly reasonable assumption in practice since the
number of domains in applications is the order of several hundreds,
while the typical number of model parameters can be significantly more
than several millions. Furthermore, by combining the cover-based bound
\eqref{eq:finite_bound} with VC-dimension bounds, one can reduce the
penalty of model selection to the following:
$\cO \left( \min \left(\frac{M\sqrt{p \log
        \frac{p}{\delta\epsilon}}}{\sqrt{m_0}}, \sqrt{\frac{d}{m_0}}
  \right)\right)$.  Let $T$ denote the time complexity of finding
$h_{\overline{\sD}_\lambda}$ for a given $\lambda$ is $T$. Then, the
overall time complexity of \ltdmsa$(\Lambda)$ is
$\left(\frac{p}{\epsilon}\right)^{p-1} T$. Thus, the algorithm is
efficient for small values of $p$.

\subsection{\boost\ algorithm} 
\label{sec:boost}

In this section, we seek a more efficient boosting-type solution
to the MSA problem that we call \boost. This consists
of considering the family of base predictors $\set{h_\lambda \colon
  \lambda \in \Lambda}$ and searching for an optimal ensemble.
The problem is therefore the following convex optimization in
terms of the mixture weights $\lambda$:
\begin{equation}
\label{eq:convex}
\min_\lambda \sL_{\sD_0} \Big( \sum_{\lambda \in \Lambda} \alpha_\lambda h_\lambda \Big),
\end{equation}
subject to $\sum_{\lambda \in \Lambda} \alpha_\lambda = 1$ and
$\alpha_\lambda \geq 0$ for all $\lambda$. 

We first show that the solution of this optimization problem benefits
from a generalization guarantee similar to that of
\ltdmsa($\Lambda$).
\begin{proposition}
\label{lem:boost}
Let $\epsilon \leq 1$ and $\ell$ be $L$ Lipschitz.  Let $\Lambda$ be
the $\epsilon$-cover defined in Theorem~\ref{thm:model}.  Then, for
any $\delta >0$, with probability at least $1 -\delta$, the solution
of \eqref{eq:convex} $h_m$ satisfies the following inequality:
\begin{align*}
\conf{&}  \sL_{\sD_0}(h_m) - \min_{h \in \sH} \sL_{\sD_0}(h) \conf{\\}
& \leq
\cE(\lambda^*)  
+ 2 \epsilon M +
L \sqrt{\frac{2p \log \frac{p}{\epsilon}}{m_0}} + 2 M \sqrt{\frac{\log \frac{1}{\delta}}{m_0}}.
\end{align*}
\end{proposition}
\begin{proof}
Let $\text{conv}(\sH_\Lambda)$ denote the convex hull of $\sH_\Lambda$. We show that
\begin{align*}
\sL_{\sD_0}(h_m) - \min_{h \in \text{conv}(\sH_\Lambda)} \sL_{\sD_0}(h)
\leq L \sqrt{\frac{2p \log \frac{p}{\epsilon}}{m_0}}.
\end{align*}
The rest of the proof is similar to that of Theorem~\ref{thm:model}
and is thus omitted. For any algorithm output $h_m$ trained
on $\h{\sD}_0$,
\conf{
\begin{multline*}
\sL_{\sD_0}(h_m) - \min_{h \in \text{conv}(\sH_\Lambda)} \sL_{\sD_0}(h) \conf{\\}
\leq 2\max_{h \in \text{conv}(\sH_\Lambda)}  \lvert \sL_{\sD_0}(h) - \sL_{\h{\sD}_0}(h) \rvert.
\end{multline*}
}
\arxiv{
\begin{align*}
\sL_{\sD_0}(h_m) - \min_{h \in \text{conv}(\sH_\Lambda)} \sL_{\sD_0}(h) \conf{\\}
\leq 2\max_{h \in \text{conv}(\sH_\Lambda)}  \lvert \sL_{\sD_0}(h) - \sL_{\h{\sD}_0}(h) \rvert.
\end{align*}
}
By McDiarmid's inequality, with probability at least $1 - \delta$,
\begin{align*}
\conf{&} \max_{h \in \text{conv}(\sH_\Lambda)}\lvert  \sL_{\sD_0}(h) - \sL_{\h{\sD}_0}(h) \rvert \conf{\\
}\conf{&} \leq 
 \E  \max_{h \in \text{conv}(\sH_\Lambda)}\lvert \sL_{\sD_0}(h) - \sL_{\h{\sD}_0}(h) \rvert
 + 
2 M \sqrt{\frac{\log \frac{1}{\delta}}{m_0}},
\end{align*}
By the definition of the Rademacher complexity,
\[
 \E  \Big[ \max_{h \in \text{conv}(\sH_\Lambda)} \lvert \sL_{\sD_0}(h) -
 \sL_{\h{\sD}_0}(h) \rvert \Big] 
\leq
 \R_{m_0}(\text{conv}(\ell(\sH_\Lambda))).
\]
Since the Rademacher complexity of a convex hull coincides with that
of the class,
\begin{align*}
\R_{m_0}(\ell(\conv(\sH_{\Lambda}))) &  \leq L \R_{m_0}(\conv(\sH_{\Lambda}))\\
& = L \R_{m_0}(\sH_{\Lambda}) \leq L \sqrt{\frac{2p \log \frac{p}{\epsilon}}{m_0}}.
\end{align*}
This completes the proof.
\end{proof}

Since the loss function is convex, \eqref{eq:convex} is convex in
$\alpha_\lambda$. However, the number of predictors is
$(p/\epsilon)^{p-1}$, which can be potentially large. This scenario is
very similar to that of boosting where the number of base
predictors such as decision trees can be very large and 
where the goal is to find a convex combination that performs well.
 To tackle this problem, we
can use randomized or block-randomized coordinate decent (RCD)
\citep{nesterov2012efficiency}. The convergence guarantees follow from
known results on RCD \citep{nesterov2012efficiency}.

Motivated by this, the algorithm proceeds as follows. Let $\lambda^t$
be the coordinate chosen at time $t$ and $\alpha_{\lambda^t}$,
$h_{\lambda^t}$ be the corresponding mixture weight and the hypothesis
at time $t$. We propose to find $\alpha_{\lambda^{t+1}}$ and
$\lambda^{t+1}$ as follows.  The algorithm randomly selects $s$ values
of $\lambda$, denoted by $S^{t+1}$ and chooses the one that minimizes
\[
\alpha^{t+1}, \lambda^{t+1} = \argmin_{\alpha, \lambda \in S^{t+1}}
\sL_{\h{\sD}_0}\left(\sum^t_{i=1} \alpha_{\lambda^i} h_{\lambda^i} + \alpha
h_\lambda\right).
\]
We refer to this algorithm as \boost. 
It is known that the the above algorithm converges to the global
optimum \citep{nesterov2012efficiency}.

In practice, for efficiency purposes, we can use different sampling
schemes.  Suppose, for example, that we have a hierarchical clustering
of $\Lambda$. At each round, instead of randomly sampling a set $S^t$
with $s$ values of $\lambda$, we could sample, $s$ values of
$\lambda$, one from each cluster and find the $\lambda$ with the
maximum decrease in loss. We can then sample $s$ values of $\lambda$,
one from each sub-cluster of the chosen cluster.  This process is
repeated till the reduction in loss is small, at which point we can
choose the corresponding $\lambda$ as $\lambda^{t+1}$. This algorithm
is similar to heuristics used for boosting with decision trees.

\subsection{\minmax\ algorithm}

Theorem~\ref{thm:model} shows algorithm \ltdmsa$(\Lambda)$
benefits from favorable guarantees for finite covers. Here, we
seek gradient-descent type solution that mimics \ltdmsa\ and is
computationally efficient. To that end, we extend this result to the
entire simplex $\Delta_p$. To prove generalization bounds for
\ltdmsa$(\Delta_p)$, we need the additional assumption that the
loss function $\ell$ is strongly convex in the parameters of
optimization.  The generalization bound uses the following lemma
proven in Appendix~\ref{app:smoothness}.
\begin{lemma}
\label{lem:smoothness}
Let $h_{\lambda} = \argmin \cL_{\overline \sD_{\lambda}}(h)$, and $\ell$ be
a $\mu$-strongly convex function whose gradient norms are bounded, $\|\nabla \ell(h(x), y))\| \leq G$ for all $x, y$.
Then for any distribution $\sD_0$,
\[
\sL_{\sD_0}(h_\lambda) - \sL_{\sD_0}(h_{\lambda'}) \leq 
\frac{G\sqrt{M}}{\sqrt{\mu}} \cdot \norm{\lambda -\lambda'}^{1/2}_1.
\]
\end{lemma}
The following lemma 
provides a generalization guarantee for \ltdmsa$(\Delta_p)$ and is prove in Appendix~\ref{app:minmax}.
\begin{lemma}
\label{thm:minmax}
Under the assumptions of Lemma~\ref{lem:smoothness}, for any
$\delta > 0$, with probability at least $1 - \delta$, the hypothesis
$h_m$ returned by \ltdmsa$(\Delta_p)$ satisfies the following
inequality:
\conf{
\begin{multline*}
\sL_{{\sD}_0}(h_m) - \min_{h \in \sH} \sL_{{\sD}_0}(h) \conf{\\}
\leq   \cE(\lambda^*) + 
\min_{\epsilon \geq 0} \frac{2 \sqrt{p \log \frac{G^2M}{\epsilon^2 \mu
  \delta}}}{\sqrt{m_0}} 
 + 2 \epsilon M.
\end{multline*}
}
\arxiv{
\begin{align*}
\sL_{{\sD}_0}(h_m) - \min_{h \in \sH} \sL_{{\sD}_0}(h) \conf{\\}
\leq   \cE(\lambda^*) + 
\min_{\epsilon \geq 0} \frac{2 \sqrt{p \log \frac{G^2M}{\epsilon^2 \mu
  \delta}}}{\sqrt{m_0}} 
 + 2 \epsilon M.
\end{align*}
}
\end{lemma}

In view of these results, we propose a gradient descent based
algorithm \minmax\ for solving the $\ltdmsa$
objective. The following is the corresponding optimization problem:
\begin{equation}
  \label{eq:minmax}
    \min_{h \in \sH, \lambda \in \Delta_p} \max_{\gamma \geq 0, h' \in
      \sH} \sL_{\h{\sD}_0}(h) + \gamma \left( \cL_{\overline
      \sD_\lambda}(h) - \cL_{\overline \sD_\lambda}(h'). \right)
\end{equation}
The above algorithm can be viewed as a two-player game, where the
first player controls the hypothesis $h$ and the weights $\lambda$ and
the second player controls the Lagrange multiplier $\gamma$ and the
alternate hypothesis $h'$. Here, the goal of the first player is to
find the best hypothesis that minimizes the best fitting model, while
the second player acts as a \emph{certifier} who determines if the
model selected by the first player belongs to $\sH_{\Delta_p}$.  We
show that \eqref{eq:minmax} returns the same solution as
\ltdmsa$(\Delta_p)$ for strictly convex functions.

\begin{theorem}
\label{lem:same}
Assume that $\ell$ is strictly convex. Then, the minimizer
of \eqref{eq:minmax} coincides with the output of
\ltdmsa$(\Delta_p)$.
\end{theorem}
\begin{proof}
If the function $\ell$ is strictly convex in $h$,
\begin{align*}
\conf{&} \min_{h \in \sH_{\Delta_p}} \sL_{\h{\sD}_0}(h)  \conf{\\}
     & = \min_{h \in \sH} \sL_{\h{\sD}_0}(h) +  \max_{\gamma \geq 0}  \gamma \indic_{h \notin \sH_{\Delta_p}} \\
     & \stackrel{(a)}{=} \min_{h \in \sH} \sL_{\h{\sD}_0}(h) +  \max_{\gamma \geq 0} \gamma \min_{\lambda \in \Delta_p} \left( \cL_{\overline \sD_\lambda}(h) - \cL_{\overline \sD_\lambda}(h_\lambda) \right) \\
      & \stackrel{(b)}{=} \min_{h \in \sH} \sL_{\h{\sD}_0}(h) + \min_{\lambda \in \Delta_p}  \max_{\gamma \geq 0}  \gamma  \left( \cL_{\overline \sD_\lambda}(h) - \cL_{\overline \sD_\lambda}(h_\lambda) \right) \\
          & = \min_{h \in \sH} \min_{\lambda \in \Delta_p}  \max_{\gamma \geq 0}  \sL_{\h{\sD}_0}(h) + \gamma  \left( \cL_{\overline \sD_\lambda}(h) - \cL_{\overline \sD_\lambda}(h_\lambda) \right) \\
                  & = \min_{h \in \sH} \min_{\lambda \in \Delta_p}  \max_{\gamma \geq 0}
                  \max_{h' \in \sH} \sL_{\h{\sD}_0}(h) + \gamma  \left( \cL_{\overline \sD_\lambda}(h) - \cL_{\overline \sD_\lambda}(h') \right),
\end{align*}
where $(a)$ follows from the fact that $\ell$ is strongly convex. For
$(b)$ we break analysis into two cases. If $h \in \sH_{\Delta_p}$, then
both $\max_{\gamma \geq 0} \gamma \min_{\lambda \in \Delta_p} \left(
\cL_{\overline \sD_\lambda}(h) - \cL_{\overline \sD_\lambda}(h_\lambda) \right)$
and $ \min_{\lambda \in \Delta_p} \max_{\gamma \geq 0} \gamma \left(
\cL_{\overline \sD_\lambda}(h) - \cL_{\overline \sD_\lambda}(h_\lambda) \right)$
are zero. 
Similarly, if $h \notin \sH_{\Delta_p}$, then both of these quantities
are infinite and can be achieved by
$\gamma \to \infty$.
This completes the proof.
\end{proof}
While the objective in \eqref{eq:minmax} is linear in $\lambda$,
convex in $\sH$, it is not jointly convex in both $\lambda$ and
$\sH$. Hence, the convergence guarantees of the minmax mirror descent
algorithm \citep{NemirovskiYudin1983} do not hold directly. However,
one can use the minmax mirror descent algorithm or stochastic minmax
mirror descent algorithms \citep{juditsky2011solving,
  namkoong2016stochastic, cotter2018training, mohri2019agnostic} to obtain heuristic solutions.  
  
To evaluate its usefulness, we first conducted experiments on a
synthetic regression example, where the ground truth is known. We fix
Let $\sX = \bR^d$, $\sY = \bR$, $p = 4$, and $d = 100$.  For each
domain $k$, $\sD_k(x)$ is distributed $N(0, \bI_d/d)$ and
$y = w^t_k x + N(0,\bI_d\sigma^2)$, where $w_k$ is distributed
according to $N(0, \bI_d/d)$ independently.  We set $\sigma^2 = 0.01$
and $\lambda^* = [0.7, 0.1, 0.1, 0.1]$.  For each source domain $k$,
we use $m_k = 10000$ examples and evaluate the results of the
algorithm as we vary $m_0$, the number of samples in the target
domain.  The results are presented in Table~\ref{tab:toy}. Observe
that the model trained only on the target dataset is significantly
worse compared to the loss when $\lambda^*$ is known.  However, it
performs nearly as well as the known mixture algorithm with 
as few as $100$ samples.

\begin{table}[t]
  \centering
  \caption{Test loss of various algorithms as a function of $m_0$. All
    losses are scaled by $1000$. The loss when $w_k$ and $\lambda^*$
    are known is $4.47$.}
  \begin{tabular}{c c c } 
    $m_0$ & $\sL_{\h{\sD}_0}$ &   \minmax   \\ \hline
     $50$ & $13.16$ & $5.15$   \\
       $100$ & $33.33$ &  $4.85$   \\
         $200$ & $9.13$ & $4.80$  \\
           $300$ & $6.73$ & $4.66$   \\
             $400$ & $6.06$ &  $4.74$   \\
  \end{tabular}
\label{tab:toy}
\vskip -.15in
\end{table}

\subsection{Lower bound}

The bounds of Theorem~\ref{thm:model} and Lemma~\ref{thm:minmax},
contain a model selection penalty of
$\cO(\sqrt{p/m_0\log(1/\epsilon)})$.  Using an information theoretic
bound, we show that any algorithm incurs a penalty of
$\Omega(\sqrt{p/m_0})$ for some problem settings.  We relegate the
proof to Appendix~\ref{app:lower}.
\begin{theorem}
\label{thm:lower}
For any algorithm $\cA$, there exists a set of hypotheses $\sH$, a loss
function $\ell$, and distributions $\sD_0, \sD_1, \sD_2,\ldots, \sD_p$,
such that $\cE(\lambda^*) = 0$ and the following holds. Given
infinitely many samples from $\sD_1, \sD_2,\ldots, \sD_p$ and $m_0$
samples from $\sD_0$, the output of the algorithm $h_\cA$ satisfies,
\[
\E[ \sL_{\sD_0}(h_\cA)] \geq \min_{h \in \sH} \sL_{\sD_0}(h) + c \cdot
\sqrt{\frac{p}{m_0}},
\]
where $c$ is a constant and the expectation is over the randomization in the algorithm and the samples.
\end{theorem}

\section{Alternative techniques}
\label{sec:disc_all}

Here, we briefly discuss some existing algorithms, in particular the competitive algorithm of
\cite{konstantinov2019robust},  which
we will compare with
our \ltdmsa\ algorithms in  experiments.

One natural approach to tackle the MSA problem we are studying
consists of using discrepancy to find $\lambda$,
by assigning a higher weight
$\lambda_k$ to a source domain $k$ that is closer to the target
distribution $\sD_0$
\citep{wen2019domain,konstantinov2019robust}. This approach
therefore relies on the estimation of the pairwise discrepancies
$\disc_{\sH}(\sD_k, \sD_0)$ between each source domain $k$ and
the target domain. Specifically, the algorithm of 
\citet{konstantinov2019robust} consists of selecting $\lambda$
by minimizing the following objective:
\[
  \sum^p_{k=1} \lambda_k \disc_{\sH}(\sD_k, \sD_0) + \gamma \sqrt{m
    \s(\lambda||{\bf m})},
\]
for some regularization parameter $\gamma$.  

We argue that this approach can be sub-optimal in various
scenarios and that the estimation of the discrepancies
in general can lead to weaker guarantees.

To illustrate this, consider the case where
the sample size is the same for all source
domains and where 
$\disc_{\sH}(\sD_0, \sD_1) = \disc_{\sH}(\sD_0, \sD_2) =
\disc_{\sH}(\sD_0, \sD_3) > 0$. Then, 
for any value $\gamma$, the
weights assigned by the algorithm coincide: 
$\lambda_1 = \lambda_2 = \lambda_3$, which is
sub-optimal for scenarios such as 
that of the following example.

\begin{figure}[t]
\centering
\includegraphics[scale=0.1]{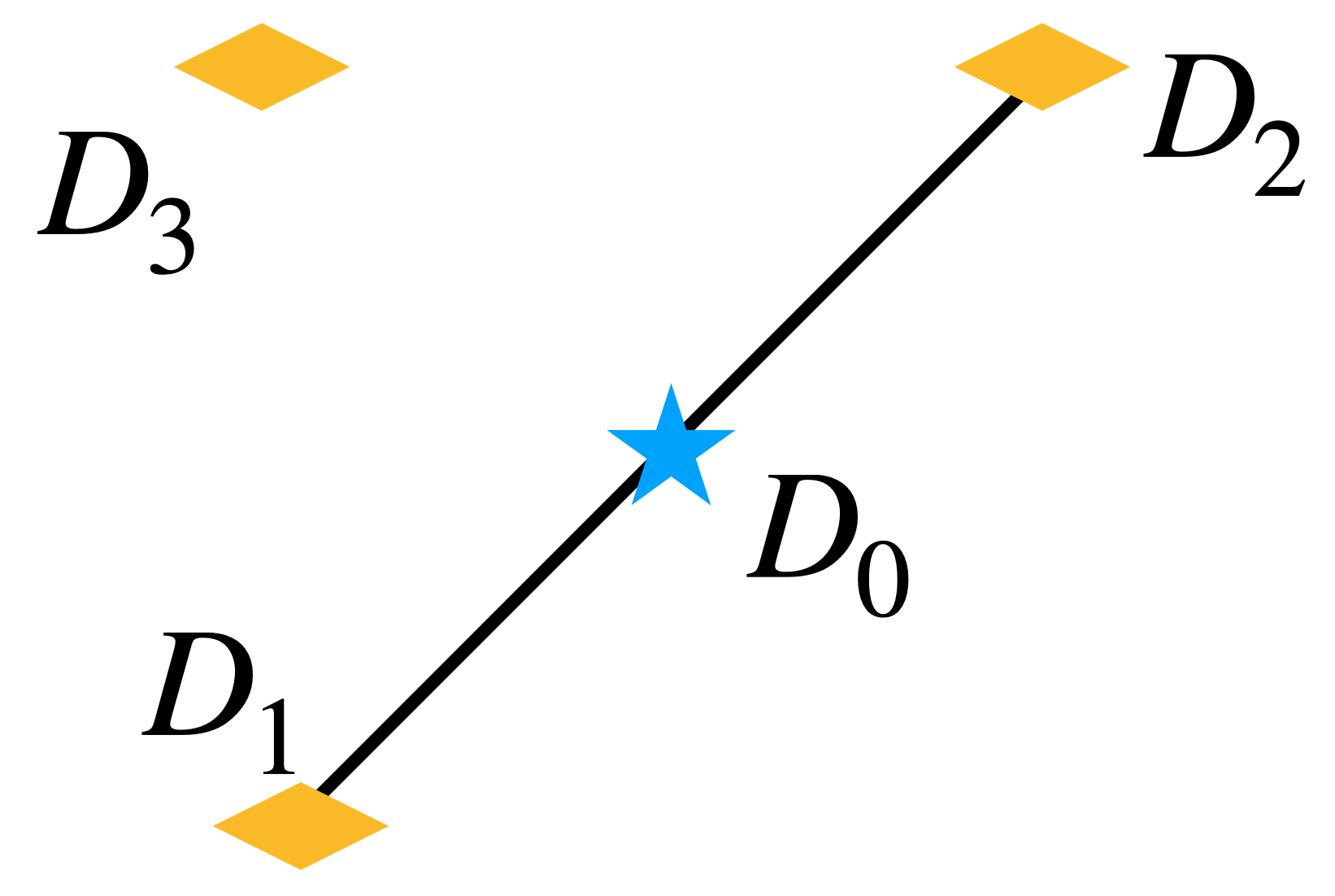}
\caption{Illustration of the pairwise discrepancy approach.}
\label{fig:pair}
\end{figure}
\begin{example}
\label{ex:pair}
Let $p = 3$ and $\sD_0 = \frac{\sD_1 + \sD_2}{2}$, with $\disc_{\sH}(\sD_0,
\sD_1) = \disc_{\sH}(\sD_0, \sD_2)= \disc_{\sH}(\sD_0, \sD_3) >0$. 
Furthermore let the number of samples from each source domain be very large. In this
case, observe that $\lambda^* = (0.5, 0.5, 0)$. If we just use the
pairwise discrepancies between $\sD_0$ and $\sD_k$ to set $\lambda$,
then $\lambda$ would satisfy $\lambda_1 = \lambda_2 = \lambda_3 = 1/3$,
which is far from optimal. The example is illustrated in Figure~\ref{fig:pair}.
\end{example}

\begin{table*}[t]
  \centering
  \caption{Test accuracy  of  algorithms for different target domains. The instances where the proposed algorithm performs better than all the baselines are highlighted.  The standard deviations are calculated over ten runs.}
  \begin{tabular}{ l c c c c } 
 algorithm & MNIST & MNIST-M & SVHN & SynthDigits \\ \hline
best-single-source & $98.0 (0.1)$ & $56.0 (0.7)$ & $83.1 (0.4)$ & $86.1 (0.4)$  \\
 combined-sources &  ${98.4}(0.1)$ & $67.2 (0.4)$ & $81.1 (0.6)$ & $87.2 (0.1)$  \\
 target-only & $96.4 (0.1)$ & $86.3 (0.5)$ &  $77.7 (0.5)$ & $88.5 (0.2)$  \\
 sources+target & $\mathbf{98.6} (0.1)$ & $74.8 (0.5)$ &  $85.4(0.3)$ & $90.6 (0.3)$ \\
 sources+target (equal weight) & $97.4 (0.2)$ & $77.8(0.6)$ &  $85.5(0.4)$ & $89.8 (0.3)$  \\
 \citep{konstantinov2019robust} &  ${98.4}(0.1)$ & $84.6 (0.5)$ & $86.3 (0.4)$ & $90.5 (0.4)$ \\
\hline
 \ltdmsa\ & $\mathbf{98.5} (0.1)$ & $\mathbf{87.6} (0.6)$ & $86.2 (0.4)$ & $\mathbf{91.5} (0.2)$ \\
 \boost\ &
${98.4} (0.2)$ & $\mathbf{88.1} (0.4)$ & $86.1 (0.4)$  & $\mathbf{91.4} (0.3)$ \\
 \minmax\ & $98.0 (0.3)$ & $\mathbf{89.5} (0.4)$ & $\mathbf{86.7} (0.4)$ & $\mathbf{91.7} (0.3)$ 
  \end{tabular}
\label{tab:image}
\vskip -.15in
\end{table*}

Since the convergence guarantees of this proposed algorithm are 
based on pairwise discrepancies, loosely speaking, the guarantees are 
tight in our formulation when $\min_{\lambda} \disc_{\sH}(\sD_0, \sD_\lambda)$ 
is close to  $\min_{\lambda} \sum_k \lambda_k \disc_{\sH}(\sD_0, \sD_k)$. 
However, for examples similar to above, such an algorithm would be sub-optimal. 

Instead of computing pairwise discrepancies, one can compute the discrepancy 
between $\sD_0$ and $\sD_\lambda$, that is $ \disc_{\sH}(\sD_0, \sD_\lambda)$, and 
choose $\lambda$ to minimize this discrepancy. However, this further 
requires estimating the discrepancy between the source and target domains 
and the generalization bound varies as 
$\tilde{\mathcal{O}}\left(\sqrt{\frac{d}{m_0}} \right)$, 
which can again be weak or uninformative for small values of $m_0$. 
We further discuss this question in more detail in Appendix~\ref{app:convex}.

\section{Experiments}
\label{sec:exp}

We evaluated our algorithms and compared them to several baselines. We state some results here and relegate additional experimental results to Appendix~\ref{app:experiments} due to space constraints.
We evaluated our algorithm on a standard MSA dataset composed of
four domains: MNIST \citep{lecun-mnisthandwrittendigit-2010},
MNIST-M \citep{ganin_icml15}, SVHN \citep{netzer2011reading}, and
SynthDigits \citep{ganin_icml15}, by treating one of MNIST, MNIST-M,
SVHN, or SynthDigits as the target domain, and the rest as source.  We
used the same preprocessing and data split as
\citep{zhao2018adversarial}, i.e., $20,000$ labeled training samples
for each domain when used as a source.  When a domain is used the
target, we used the first $1280$ examples from the $20,000$.  We also
used the same convolution neural network as the digit classification
model in \citep{zhao2018adversarial}, with the exception that we used
a regular ReLU instead of leaky ReLU.  Unlike
\citep{zhao2018adversarial}, we trained the models using stochastic
gradient descent with a fixed learning rate without weight decay.

We used several baselines for comparison: 
\begin{itemize}
\item[$(i)$] \emph{best-single-source}:
best model trained only on one of the sources; 
\item[$(ii)$] \emph{combined-sources}: model trained 
on dataset obtained by concatenating all the sources; 
\item[$(iii)$] \emph{target-only}: model
trained only on the limited target data; 
\item [$(iv)$]
\emph{sources+target}: models trained by combining source and
targets; 
\item[$(v)$] \emph{sources + target (equal weight)}: models
trained by combining source and targets where all of them get the same
weight; 
\item[$(vi)$] \emph{pairwise discrepancy}: the pairwise
discrepancy approach of \citet{konstantinov2019robust}.
\end{itemize}

Baselines
$(ii)$, $(iv)$, and $(v)$ involve data concatenation.  For 
baseline $(vi)$ and the proposed algorithms \ltdmsa,
\boost, \minmax, we report the better
results of the following two approaches: one where all $1280$ target
samples are treated as $\h{\sD}_0$ and one where $1024$ random 
samples
are treated as a separate new source and $256$ samples are treated as
samples from $\h{\sD}_0$.

%\begin{wrapfigure}{r}{0.4\textwidth}
\begin{figure}[t]
\centering
\includegraphics[width=\arxiv{0.4}\conf{0.66}\linewidth]{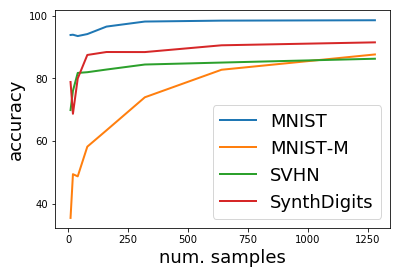}
\caption{Performance of \ltdmsa\ as a function 
of the target sample size $m_0$.}
\label{fig:curves}
\vskip -.15in
\end{figure}
%\end{wrapfigure}

The results are presented in Table~\ref{tab:image}. Our \ltdmsa\ 
algorithms
perform well compared to the baselines. We note that
\minmax\ performed better using all $1280$ target samples
as $\h{\sD}_0$, whereas \citet{konstantinov2019robust}, \ltdmsa,
and \boost\ performed better using $1024$ target samples
as a separate new source domain.  As expected, the performance of
proposed algorithms is better than that of the unsupervised domain
adaptation algorithms of \citep{zhao2018adversarial} (see Table 2 in
their paper), due to the availability of labeled target samples.

Figure~\ref{fig:curves} shows the performance of the \ltdmsa\ as
a function of the number of target samples.  Of the four target
domains, MNIST is the easiest domain and requires very few target
samples to achieve good accuracy, and MNIST-M is the hardest and
requires many target samples to achieve good accuracy.  We omit the
curves for \boost\ and \minmax\ because they
are similar. 

\section{Conclusion}

We presented a theoretical and algorithmic study
of multiple-source domain adaptation with limited target
labeled data. The algorithms we presented benefit from
very favorable learning guarantees and further perform
well in our experiments, typically surpassing other
baselines. We hope that our analysis will serve as a tool
for further theoretical studies of this problem and other
related adaptation problems and algorithms.

\newpage
\bibliography{cmsa}
\conf{\bibliographystyle{abbrvnat}}

\newpage
\appendix

\conf{
\onecolumn
\begin{center}
    {\Large{ Supplementary material: A Theory of Multiple-Source Adaptation with Limited Target Labeled Data}}
\end{center}
}

\section{Related on domain adaptation}
\label{app:related}

As stated in the introduction, various scenarios of adaptation can be distinguished depending on parameters such as the number of source domains available, the presence or absence of target labeled data, and access to labeled source data or only to predictors trained on each source domain.
Single source domain adaptation has been studied in several papers including \citep{KiferBenDavidGehrke2004, ben2010theory,  MansourMohriRostamizadeh2009Bis}. 
 
Several algorithms have been proposed for multiple-source adaptation. \cite{khosla2012undoing,blanchard2011generalizing} proposed to combine all the source data and train a single model. \cite{duan2009domain,duan2012domain} used unlabeled target data to obtain a regularizer. Domain adaptation via adversarial learning was studied by 
\cite{multiadversarial_aaai2018,zhao2018adversarial}. 
\cite{crammer2008learning} considered learning models for
each source domain, using close-by data of other domains.
\cite{gong_cvpr12} ranked multiple source domains by how well they can
adapt to a target domain. Other solutions to multiple-source domain adaptation include, clustering \citep{liu2016structure}, learning domain-invariant features \citep{gong_icml13}, learning intermediate representations \citep{jhuo2012robust}, subspace alignment techniques \citep{fernando2013unsupervised}, attributes detection \citep{gan2016learning}, using a linear combination of pretrained classifiers \citep{yang_acmm07}, using multitask auto-encoders \citep{ghifary2015domain}, causal approaches \citep{sun2011two}, two-state weighting approaches \citep{sun2011two}, moments alignment techniques \citep{peng2019moment} and domain-invariant component analysis  \citep{MuandetBalduzziScholkopf2013}.

\section{Proof of equation \eqref{eq:disc_lambda}}
\label{app:disc_lambda}
By the definition of discrepancy,
\[
\sL_{\sD_0} (h_{\overline{\sD}_\lambda}) 
\leq \sL_{{\sD_\lambda}} (h_{\overline{\sD}_\lambda})
+ \disc_{\sH}(\sD_\lambda, \sD_0).
\]
Similarly,
\begin{align*}
\sL_{\sD_0} (h_{{\sD}_0})
& = \sL_{\sD_0} (h_{{\sD}_0}) -  \sL_{\sD_\lambda} (h_{{\sD}_0})+ \sL_{\sD_\lambda} (h_{{\sD}_0}) \\
& \geq  \sL_{\sD_\lambda} (h_{{\sD}_0}) - \disc_{\sH}(\sD_\lambda, \sD_0) \\
& \geq  \sL_{\sD_\lambda} (h_{{\sD}_\lambda}) - \disc_{\sH}(\sD_\lambda, \sD_0) \\
\end{align*}
Combining the above two equations yields
\[
\sL_{\sD_0} (h_{\overline{\sD}_\lambda}) - \sL_{\sD_0} (h_{{\sD}_0})  \leq \sL_{{\sD_\lambda}} (h_{\overline{\sD}_\lambda})  -  \sL_{{\sD_\lambda}} (h_{{\sD}_\lambda}) + 2 \disc_{\sH}(\sD_\lambda, \sD_0).
\]

Next observe that, by rearranging terms,
\begin{align*}
\sL_{{\sD_\lambda}} (h_{\overline{\sD}_\lambda})  -  \sL_{{\sD_\lambda}} (h_{{\sD}_\lambda})
& =  \sL_{{\sD_\lambda}} (h_{\overline{\sD}_\lambda}) -  \sL_{\overline{\sD}_\lambda} (h_{\overline{\sD}_\lambda}) + 
\sL_{\overline{\sD}_\lambda} (h_{{\sD}_\lambda}) - \sL_{{\sD_\lambda}} (h_{{\sD}_\lambda}) \\
& + 
 \sL_{\overline{\sD}_\lambda} (h_{\overline{\sD}_\lambda}) - \sL_{\overline{\sD}_\lambda} (h_{{\sD}_\lambda}) 
\end{align*}
However, by the definition of $h_{\overline{\sD}_\lambda}$, 

\begin{align*}
 \sL_{\overline{\sD}_\lambda} (h_{\overline{\sD}_\lambda}) \leq  \sL_{\overline{\sD}_\lambda} (h_{{\sD}_\lambda}).
 \end{align*}
Hence,
\begin{align*}
 \sL_{{\sD_\lambda}} (h_{\overline{\sD}_\lambda})  -  \sL_{{\sD_\lambda}} (h_{{\sD}_\lambda})
& \leq  \sL_{{\sD_\lambda}} (h_{\overline{\sD}_\lambda}) -  \sL_{\overline{\sD}_\lambda} (h_{\overline{\sD}_\lambda}) + 
\sL_{\overline{\sD}_\lambda} (h_{{\sD}_\lambda}) - \sL_{{\sD_\lambda}} (h_{{\sD}_\lambda}) \\
& \leq 2 \sup_{h} |\sL_{\overline{\sD}_\lambda}(h)  -  \sL_{{\sD_\lambda}} (h)|,
\end{align*}
where the last inequality follows by taking the supremum. Combining the above equations, gives the proof. 

\section{Previous work}

\subsection{Bregman divergence based non-negative matrix factorization}
\label{app:bregman}

A natural algorithm is a two step
process, where we first identify the optimal $\lambda$ by minimizing
\[
\min_{\lambda \in \Delta_p} \sfB(\h{\sD}_0 || \overline{\sD}_\lambda),
\]
where $\sfB$ is a suitable Bregman divergence.  We can then use
$\lambda$ to minimize the weighted loss. However, this approach has
both practical and theoretical issues. On the practical side, if $\cX$
is a continuous space, then the empirical distribution
$\overline \sD_\lambda$ would be a point mass distribution over
observed points and would never converge to the true distribution
$\sD_\lambda$. To overcome this, we need to first use
$\overline{\sD}_\lambda$ to estimate the distribution $\sD_\lambda$
via kernel density estimation or other methods and then use the
estimate instead of $\overline{\sD}_\lambda$. Even if we use these
methods and find $\lambda$, it is likely that we would overfit as the
generalization of the algorithm depends on the covering number of
$\{\sD_\lambda : \lambda \in \Delta_p\}$, which in general can be much
larger than that of the class of hypotheses $\sH$. Hence such an
algorithm would not incur generalization loss of $\cE(\lambda^*)$.
One can try to reduce the generalization error by using a discrepancy
based approach, which we discuss next.

\subsection{A convex combination discrepancy-based algorithm}
\label{app:convex}

Since pairwise discrepancies would result in identifying a sub-optimal
$\lambda$, instead of just considering the pairwise discrepancies, one
can consider the discrepancy between $\sD_0$ and any
$\sD_\lambda$. Since
\begin{equation*}
\label{eq:one}
\sL_{\sD_0}(h) \leq 
\min_{\lambda \in \Delta_p} \sL_{\sD_\lambda}(h) + \disc_{\sH}(\sD_0, \sD_\lambda),
\end{equation*}
and the learner has more data from $\sD_\lambda$ than from $\sD_0$, a
natural algorithm is to minimize
$\sL_{\sD_\lambda}(h) + \disc_{\sH}(\sD_0, \sD_\lambda)$.  However,
note that this requires estimating both the discrepancy
$\disc_{\sH}(\sD_0, \sD_\lambda)$ and the expected loss over
$\sL_{\sD_\lambda}(h)$.  In order to account for both terms, we
propose to minimize the upper bound on
$\min_{\lambda \in \Delta_p} \sL_{\sD_\lambda}(h) + \disc_{\sH}(\sD_0,
\sD_\lambda)$,
\begin{align}
\label{eq:disc_opt}
\min_{\lambda} \sL_{\overline \sD_\lambda}(h) + C_{\epsilon}(\lambda),
 \end{align}
 where $C_{\epsilon}(\lambda)$ is given by,
 \begin{align*}
\conf{&} \disc_{\sH}(\h{\sD}_0, \overline \sD_\lambda)
+ \frac{c\sqrt{d + \log \frac{1}{\delta}}}{\sqrt{m_0}}+ \epsilon M  + 
\frac{cM\sqrt{\s(\lambda || \bm)}}{\sqrt{m}} \cdot 
\left(\sqrt{d \log \frac{em}{d} + p \log  \frac{1}{\epsilon\delta}} \right),
\end{align*}
for some constant $c$.  We first show that right hand side of
 \eqref{eq:disc_opt} is an upper bound on $\sL_{\sD_0}(h)$.
\begin{lemma}
\label{lem:c_bound}
With probability at least $1-2\delta$, for all $h \in \sH$ and
$\lambda \in \Delta_p$,
\begin{align*}
|     \sL_{\sD_0}(h) - \sL_{\overline \sD_\lambda}(h)  
    | \leq   C_{\epsilon}(\lambda).
\end{align*}
\end{lemma}
\begin{proof}

By \eqref{eq:disc_lambda}, \eqref{eq:optimal}, and Proposition~\ref{lem:known_lambda}, with probability at least $1-\delta$,
\begin{align*}
\alsoignore{&} | \sL_{\sD_0}(h)  - \sL_{\overline \sD_\lambda}(h) | \leq 2\disc_{\sH}(\sD_0, \sD_\lambda)  \alsoignore{\\}
& + \frac{4M\sqrt{s(\lambda || \bm)}}{\sqrt{m}} \cdot \left(\sqrt{d \log \frac{em}{d} +  \log \frac{1}{\delta}} \right).
\end{align*}
Hence, by the union bound over an $\epsilon$-$\ell_1$ cover of $\Delta_p$
yields, with probability $\geq 1-\delta$, for all $\lambda \in \Delta_p$,
\begin{align*}
| \sL_{\sD_0}(h)  - \sL_{\overline \sD_\lambda}(h) |
  \leq  2\disc_{\sH}(\sD_0, \sD_\lambda) + \epsilon M  
 + \frac{4M\sqrt{\s(\lambda || \bm)}}{\sqrt{m}} \cdot \left(\sqrt{d \log \frac{em}{d} + p \log \frac{1}{\epsilon\delta}} \right).
\end{align*}
With probability at least $
1-\delta$, discrepancy can be estimated as 
\begin{align*}
&|\disc_{\sH}(\sD_0, \sD_\lambda) - \disc_{\sH}(\h{\sD}_0, \overline \sD_\lambda)| \leq  \epsilon M
\\ &  + \frac{c\sqrt{d + \log \frac{1}{\delta}}}{\sqrt{m_0}} +  \frac{cM\sqrt{\s(\lambda || \bm)}}{\sqrt{m}} \cdot \left(\sqrt{d \log \frac{em}{d} + p \log \frac{1}{\epsilon\delta}} \right),
\end{align*}
for some constant $c > 0$. Combining the above equations yields, with probability at least $1 -2\delta$,
\begin{align*}
 \max_{\lambda} | \sL_{\sD_0}(h) -  \sL_{\overline \sD_\lambda}(h)  | 
 \leq C_{\epsilon}(\lambda).
\end{align*}
\end{proof}

Let $h_R$ be the solution to \eqref{eq:disc_opt}, we now give a
generalization bound for the above algorithm.
\begin{lemma}
\label{lem:disc}
With probability at least $1 -2\delta$, the solution $h_R$ for \eqref{eq:disc_opt} satisfies
\[
\sL_{\sD_0}(h_R) \leq \min_{h \in \sH} \sL_{\sD_0}(h) + 2\min_{\lambda}  C_{\epsilon}(\lambda).
\]
\end{lemma}
% \subsection{Proof of Lemma~\ref{lem:disc}}
% \label{app:disc}
%We now have the tools to prove the lemma.
\begin{proof}%[Proof of Lemma~\ref{lem:disc}]
By Lemma~\ref{lem:c_bound}, with probability at least $1 - 2\delta$,
\begin{align*}
 \min_{\lambda \in \Delta_p}    |\sL_{\sD_0}(h)  
-  \sL_{\overline \sD_\lambda}(h)| \leq \min_{\lambda \in \Delta_p} C_{\epsilon}(\lambda).
\end{align*}
Let $h_R$ be the output of the algorithm and $h_{\sD_0}$ be the minimizer of $\sL_{\sD_0}(h)$.
\begin{align*}
 \alsoignore{&} L_{\sD_0}(h_R) - \sL_{\sD_0}(h_{\sD_0}) \alsoignore{\\}
& \leq 
\min_{\lambda}  \sL_{\overline \sD_\lambda}(h_R) + C_{\epsilon}(\lambda)
- \max_{\lambda}  \sL_{\overline \sD_\lambda}(h_{\sD}) - C_{\epsilon}(\lambda) \\
& \leq 
\min_{\lambda}  \sL_{\overline \sD_\lambda}(h_R) + C_{\epsilon}(\lambda) \alsoignore{\\&}
+ \min_{\lambda} - \sL_{\overline \sD_\lambda}(h_{\sD_0}) + C_{\epsilon}(\lambda) \\
& \leq 
\min_{\lambda}  \sL_{\overline \sD_\lambda}(h_R) +C_{\epsilon}(\lambda)
 - \sL_{\overline \sD_\lambda}(h_{\sD_0}) + C_{\epsilon}(\lambda) \\
& \leq 2 \min_{\lambda}  C_{\epsilon}(\lambda),
\end{align*}
where the last inequality follows from the fact that $h_R$ is the minimizer of \eqref{eq:disc_opt}.

\end{proof}
The above bound is comparable to the model trained on only target data
as $C_{\epsilon}(\lambda)$ contains $\cO \left(\sqrt{\frac{d}{m_0}}\right)$,
which can be large for a small values of $m_0$. This bound can be
improved on certain favorable cases when $\sD_0 = \sD_k$ for some
known $k$. In this case if we use the same set of samples for
$\h{\sD}_0$ and $\h{\sD}_k$, then the bound can be improved to $\cO
\left( \frac{\sqrt{d(1-\lambda_k)}}{\sqrt{m_0}} \right)$, which in
favorable cases such that $\lambda_k$ is large, 
yields a better bound than the target-only model.  

\section{Proofs for the proposed algorithms}

\subsection{Proof of Lemma~\ref{lem:smoothness}}
\label{app:smoothness}

By the strong convexity of $\ell$,
\begin{align*}
\alsoignore{&} \sL_{\overline \sD_{\lambda}}(h_{\lambda'}) - \sL_{\overline \sD_{\lambda}}(h_{\lambda})  \alsoignore{\\}
& \geq \nabla \sL_{\overline \sD_{\lambda}}(h_{\lambda}) \cdot (h_{\lambda'} - h_{\lambda}) 
+ \frac{\mu}{2} \norm{h_{\lambda'} - h_{\lambda}}^2 \\
& = \frac{\mu}{2} \norm{h_{\lambda'} - h_{\lambda}}^2,
\end{align*}
where the equality follows from the definition of $h_{\lambda}$.
Similarly, since the function $\ell$ is bounded by $M$
\begin{align*}
\alsoignore{&} \sL_{\overline \sD_{\lambda}}(h_{\lambda'}) - \sL_{\overline
  \sD_{\lambda}}(h_{\lambda})  \alsoignore{\\}
& \leq \sL_{\overline
  \sD_{\lambda'}}(h_{\lambda'}) - \sL_{\overline
  \sD_{\lambda'}}(h_{\lambda}) + \norm{\lambda-\lambda'}_1 M \\
& \leq -\nabla \sL_{\overline \sD_{\lambda'}}(h_{\lambda'}) \cdot (h_{\lambda'} - h_{\lambda}) 
 - \frac{\mu}{2} \norm{h_{\lambda'} - h_{\lambda}}^2  + \norm{\lambda - \lambda'}_1 M \\
&  =  - \frac{\mu}{2} \norm{h_{\lambda'} - h_{\lambda}}^2  + \norm{\lambda - \lambda'}_1 M.
\end{align*}
Combining the above equations, 
\[
\mu \norm{h_{\lambda'} - h_{\lambda}}^2 \leq M \norm{\lambda -\lambda'}_1.
\]
Hence for any distribution $\sD_0$,
\begin{align*}
|\sL_{\sD_0}(h_{\lambda'}) - \sL_{\sD_0}(h_{\lambda})| 
& \leq |\nabla \sL_{\sD_0}(h_{\lambda}) \cdot (h_{\lambda'} - h_{\lambda})| \\
& \leq |\nabla \sL_{\sD_0}(h_{\sD}) \cdot (h_{\lambda'} - h_{\lambda})| \\
& = G \norm{h_{\lambda'} - h_{\lambda}} \\
& = \frac{G\sqrt{M}}{\sqrt{\mu}} \cdot \norm{\lambda -\lambda'}^{1/2}_1.
\end{align*}

\subsection{Proof of Lemma~\ref{thm:minmax}}
\label{app:minmax}
%\begin{proof}
Let $\Lambda$ be the minimal cover of $\Delta_p$ in the $\ell_1$
distance such that any two elements of the cover has distance at most
$\frac{\mu\epsilon^2}{G^2M} $. Such a cover will have at most $\left(
\frac{G^2M}{\mu\epsilon^2} \right)^p$ elements. Hence, by
Lemma~\ref{lem:smoothness}, McDiarmid's inequality, together with
union bound over the above cover, we get with probability at least $1
- \delta$,
\begin{align}
\conf{&} \sL_{{\sD}_0}(h_m) - \min_{h \in \sH_{\Delta_p}} \sL_{{\sD}_0}(h) 
\conf{\nonumber\\}
 & \leq 2 
\max_{h \in \sH_{\Delta_p}} \sL_{\sD_0}(h) - \sL_{\h{\sD}_0}(h) \nonumber\\
& \leq 2 \max_{h \in \sH_{\Lambda}} \sL_{\sD_0}(h) - \sL_{\h{\sD}_0}(h)
+ 2\epsilon M \nonumber \\
& \leq \frac{2M \sqrt{p \log \frac{G^2M}{2\epsilon^2 \mu\delta}}}{\sqrt{m_0}} + 2\epsilon M. \label{eq:temp_4}
\end{align}
Similar to the proof of Theorem~\ref{thm:model},
\begin{equation}
\label{eq:temp_5}
\min_{h \in \sH_{\Delta_p}} \sL_{\sD_0}(h) - 
\min_{h \in \sH} \sL_{\sD_0}(h) 
\leq 
\cE(\lambda).
\end{equation}
Combining \eqref{eq:temp_4} and \eqref{eq:temp_5}
and taking minimum over $\lambda$, yields the theorem.
%\end{proof}
\section{Proof of Theorem~\ref{thm:lower}}
\label{app:lower}

Let $p$ be a multiple of four. Let $\cX = \{1,2,\ldots, p/2\}$ and 
$\sY = \{0,1\}$. For all $k \leq p/2$, and $x \in \cX$, let $\sD_k(x)
= \frac{2}{p}$.  For every even $k$, let $\sD_k(1 | \lceil k/2 \rceil) = 1
$ and for every odd $k$, $\sD_k(1 | \lceil k/2 \rceil) = 0$.  For
remaining $x$ and $k$, let $\sD_k(1|x) = \frac{1}{2}$.

Let $\sH$ be the set of all mappings from $\cX \to \sY$ and the loss
function be zero-one loss.  Let $\sD_0 = \sD_\lambda$ for some
$\lambda$.  Hence, the optimal estimator $h^*$ is
\[
h^*_\lambda(1 | x) = \indic_{\lambda_{2x} > \lambda_{2x-1}}.
\]
Given infinitely number of samples from each $\sD_k$, the learner
knows the distributions $\sD_k$. Hence, roughly speaking the
algorithm has to find if $\lambda_{2x} > \lambda_{2x-1}$ for each
$x$.  

Let $\epsilon = \frac{1}{100} \cdot \sqrt{\frac{p}{m_0}}$.  We restrict $\lambda \in \Lambda$, where $\Lambda$ is defined as follows.  Let
$\Lambda$ be the set of all distributions such that for each $\lambda
\in \Lambda$ and $x$,
\[
\lambda_{2x} + \lambda_{2x-1} = \frac{2}{p},
\]
and $\lambda_{2x} \in \{\frac{1+\epsilon}{p}, \frac{1-\epsilon}{p}\}$.
Note that $|\Lambda| = 2^{p/4}$. For $x \leq p/4$, let
$s_x = \{2x, 2x-1\}$. Let $m_{s_x}$ be the number of occurrences of
elements from $s_x$.  Given $m_{s_x}$, $m_{2x}$ and $m_{2x-1}$ are
random variables from Binomial distribution with parameters $m_{s_x}$
and $\frac{\lambda_{2x}}{\lambda_{2x} + \lambda_{2x-1}}$.  This
reduces the problem of learning the best classifier into testing $p/2$
Bernoulli distributions and we can use standard tools from information
theory such as Fano's inequality \citep{cover2012elements} to provide
a lower bound. We provide a proof sketch.

Since $\sum^{p/4}_{x=1} m_{s_x} = m_0$, there are at least $p/8$
values of $s_x$ for which $m_{s_x} \leq 8m_0/p$. Consider one such
$s_x$, where $m_{s_x} \leq 8m_0 /p$. For that $x$, given $m_{s_x}$
samples from $s_x$, by Fano's inequality, with probability at least
$1/4$, any algorithm cannot differentiate between
$\lambda_{2x} > \lambda_{2x-1}$ and $\lambda_{2x} <
\lambda_{2x-1}$. Thus, with probability at least $1/4$, any algorithm
incorrectly finds the wrong hypothesis for $h$, and hence,
\[
\E[\sL(h) | x \in s_x ] \geq \E[\sL(h^*_\lambda)|  x \in s_x] + c |\lambda_{2x} - \lambda_{2x-1}| \geq \E[\sL(h^*_\lambda)|  x \in s_x] + c\epsilon,
\] 
for some constant $c$. Averaging over all symbols $x$, yields
\begin{align*}
\E[\sL(h)] 
& = \sum_{s_x} \sD_\lambda( s_x) \E[\sL(h) | x \in  s_x] \\
& = \sum_{s_x : m_{s_x} \leq 8m_0/p} \frac{2}{p} \E[\sL(h) |  x \in s_x] + 
 \sum_{s_x : m_{s_x} > 8m_0/p} \frac{2}{p} \E[\sL(h) |  x \in s_x] \\
& \geq \sum_{s_x : m_{s_x} \leq 8m_0/p} \frac{2}{p}  \E[\sL(h^*_\lambda)|  x \in s_x] + c\epsilon +   \sum_{s_x : m_{s_x} > 8m_0/p} \frac{2}{p} \E[\sL(h^*_\lambda)|  x \in s_x] \\
& =  \E[\sL(h^*_\lambda)] + \sum_{s_x : m_{s_x} \leq 8m_0/p} \frac{2}{p} c\epsilon\\
& \geq \E[\sL(h^*_\lambda)] + \frac{p}{8} \cdot  \frac{2}{p} c\epsilon \\
& = \E[\sL(h^*_\lambda)] + \frac{ c\epsilon}{4} = \E[\sL(h^*_\lambda)] + \frac{ c\epsilon}{400} \sqrt{\frac{p}{m_0}}.
\end{align*}

\section{Additional experiments}
\label{app:experiments}

In addition to the digit recognition task, we considered the standard visual adaptation Office dataset \citep{saenko_eccv10}, which has 3 domains: amazon, dslr, and webcam. This dataset consists of 31 categories of objects commonly
found in an office environment. The amazon domain consists of 2817 images, dslr 498, and webcam 795, for a total of 4110 images. 
For source domains, we used all available samples, and for target domains, we used 20 samples per category for amazon and 8 for both dslr and webcam. 
However, rather than AlexNet, we used the ResNet50 \citep{he2015deep}
architecture pre-trained on ImageNet.

Similar to the digits experiment, for baseline $(vi)$ and the proposed algorithms \ltdmsa, \boost, \minmax, we report the better
results of the following two approaches: one where all target
samples are treated as $\h{\sD}_0$ and one where some percentage of random  samples are treated as a separate new source and the remaining samples are treated as samples from $\h{\sD}_0$. For the latter approach, due to the limited size of the Office dataset, we used cross validation with 5 different splits to determine what percentage of samples to treat as a separate new source. As discussed in Appendix~\ref{app:convex}, empirical estimates of the discrepancy based on small samples are unreliable for small datasets and large model classes. Our experiments corroborated this theory. Since the Office dataset is small, the ResNet50 architecture has many parameters, and our loss (log-loss) is unbounded, the empirical pairwise discrepancy estimate was infinite. Hence, we omit the results for the pairwise discrepancy approach of \citep{konstantinov2019robust}.

The results are presented in Table~\ref{tab:office}. Our \ltdmsa\ algorithms perform well compared to the baselines.

\begin{table*}[t]
  \centering
  \caption{Test accuracy of algorithms for different target domains for the Office dataset. The instances where the proposed algorithm performs better than all the baselines are highlighted. 
  }
  \begin{tabular}{ l c c c} 
 algorithm & amazon & dslr & webcam  \\ \hline
best-single-source & $58.8 (1.0)$ & $98.7 (0.6)$ & $94.0 (1.3)$  \\
 combined-sources &  $62.0 (0.7)$ & $97.0 (0.9)$ & $91.9 (1.3)$  \\
 target-only & $77.8 (0.8)$ & $96.4 (0.8)$ &  $91.5 (0.9)$  \\
 sources+target & $77.7 (0.6)$ & $98.8 (0.6)$ &  $96.6 (0.7)$  \\
 sources+target (equal weight) & $76.7 (0.5)$ & $99.4 (0.5)$ &  $96.8 (0.6)$  \\
\hline
 \ltdmsa\ & $\mathbf{78.1} (0.5)$ & $\mathbf{99.5} (0.4)$ & $\mathbf{97.5} (0.8)$  \\
 \boost\ &
$\mathbf{78.6} (0.4)$ & $\mathbf{99.5} (0.5)$ & $\mathbf{97.6} (0.3)$  \\
 \minmax\ & $77.7 (0.7)$ & $98.9 (0.3)$ & $\mathbf{97.0} (0.5)$ 
  \end{tabular}
\label{tab:office}
\vskip -.15in
\end{table*}

\end{document}